\documentclass{tlp}
\usepackage{url,graphicx,latexsym,amssymb}

\newcommand{\ST}{\textup{\textsf{ST}}}
\newcommand{\otter}{\textsc{otter}}
\newcommand{\mace}{\textsc{mace-2}}
\newtheorem{lemma}{Lemma}
\newtheorem{theorem}{Theorem}
\newtheorem{definition}{Definition}

\title{Checking the Quality of Clinical Guidelines using Automated Reasoning Tools}

\author[Arjen Hommersom, Peter J.F. Lucas, and Patrick van Bommel]
{ARJEN HOMMERSOM, PETER J.F. LUCAS, and PATRICK VAN BOMMEL \\
Section on Model-based System Development \\
Institute for Computing and Information Sciences \\
Radboud University Nijmegen \\
PO Box 9010, 6500 GL Nijmegen \\ 
The Netherlands \\
\email{\{arjenh,peterl,pvb\}@cs.ru.nl}}

\submitted {22 August 2006}
\revised {17 July 2007, 20 April 2008}
\accepted {19 May 2008}

\begin{document}

\maketitle

\begin{abstract}
  \noindent
  Requirements about the quality of clinical guidelines can be
  represented by schemata borrowed from the theory of abductive
  diagnosis, using temporal logic to model the time-oriented aspects
  expressed in a guideline. Previously, we have shown that these
  requirements can be verified using interactive theorem proving
  techniques. In this paper, we investigate how this approach can be
  mapped to the facilities of a resolution-based theorem prover,
  \otter, and a complementary program that searches for finite models
  of first-order statements, \mace. It is shown that the reasoning
  required for checking the quality of a guideline can be mapped to
  such fully automated theorem-proving facilities. The medical quality
  of an actual guideline concerning diabetes mellitus 2 is
  investigated in this way.
\end{abstract}


\section{Introduction}

Health-care is becoming more and more complicated at an astonishing
rate. On the one hand, the number of different patient management
options has risen considerably during the last couple of decades,
whereas, on the other hand, medical doctors are expected to take
decisions balancing benefits for the patient against financial costs.
There is a growing trend within the medical profession to believe that
clinical decision-making should be based as much as possible on sound
scientific evidence; this has become known as \emph{evidence-based
medicine} \cite{woolf2000}. Evidence-based medicine has given a
major impetus to the development of clinical \emph{guidelines}, documents
offering a description of steps that must be taken and considerations
that must be taken into account by health-care professionals in
managing a disease in a patient, to avoid substandard practices or
outcomes.  Their general aim is to promote standards of medical care.
Clinical \emph{protocols} have a similar aim as clinical guidelines,
except that they offer more detail, and are often local, more
detailed version of a related clinical guideline.

Researchers in artificial intelligence (AI) have picked up on these
developments by designing guideline modelling languages, for instance
PRO\emph{forma} \cite{fox2000} and GLIF3 \cite{Peleg:2000AMIA} that
may be useful in developing computer-based representations of
guidelines.  Some of them, for example in the Asgaard project
\cite{shahar98}, in the CREDO project \cite{fox2006} and the GLARE
project \cite{TerenzianiAIMED2001,TerenzianiAMIA2003}, are also
involved in the design of tools that support the deployment of
clinical guidelines.  These languages and tools have been evolving
since the 1990s, a process that is gaining momentum due to the
increased interest in guidelines within the medical community.  AI
researchers see guidelines as good real-world examples of highly
structured, systematic documents that are amenable to formalisation.

Compared to the amount of work that has been put into the
formalisation of clinical guidelines, verification of guidelines has
received relatively little attention.  In \cite{Shiffman:1994MDM},
logic was used to check whether a set of recommendations is complete,
to find out whether or not the recommendations are logically
consistent, and to recognise ambiguous rules if they are present.
Checking the consistency of temporal scheduling constraints has also
been investigated \cite{Duftschmid:2002AIM}. Most of the work done in
the area of formal verification of clinical guidelines, i.e., proving
correctness properties using formal methods, is of more recent years,
e.g., as done in the Protocure
project\footnote{\url{http://www.protocure.org} [Accessed: 21 May 
2008]} with the use of
interactive theorem proving \cite{HommersomIEEE2007,Teije:2006AIM} and
model checking \cite{Baumler:2006SPIN,GrootAIME2007}.

This paper explores the use of \emph{automated deduction} for the
verification of clinical guidelines. For the rapid development of good
quality guidelines it is required that guidelines can be at least
partially verified automatically; unfortunately, as of yet, there are
no verification methods that can be readily used by guideline
developers.  Previously, it was shown that for reasoning about models
of medical knowledge, for example in the context of medical expert
systems \cite{lucas93}, classical automated reasoning techniques
(e.g., \cite{robinson65,wos84}) are a practical option.  Important for
the reasoning about knowledge in clinical guidelines is its temporal
nature; time plays a part in the physiological mechanisms as well as
in the exploration of treatment plans. As far as we know, the
application of automated reasoning techniques to guideline knowledge
has as yet not been investigated. The guideline we studied to this
purpose has a structure similar to other guidelines and the
verification principles used have sufficient generality. Thus, the
results of the study go beyond the actual guideline studied.

There are two approaches to checking the quality of clinical
guidelines using formal methods: (1) the \emph{object-level} approach
amounts to translating a guideline to a formal language, such as Asbru
\cite{shahar98}, and subsequently applying program verification or
logical methods to analyse the resulting representation for
establishing whether certain domain-specific properties hold; (2) the
\emph{meta-level} approach, which consists of formalising general
requirements to which a guideline should comply, and then
investigating whether this is the case. Here we are concerned with the
meta-level approach to guideline-quality checking.  For example, a
good-quality clinical guideline regarding treatment of a disorder
should preclude the prescription of redundant drugs, or advise against
the prescription of treatment that is less effective than some
alternative.  An additional goal of this paper is to establish how
feasible it is to implement such meta-reasoning techniques in existing
tools for automated deduction for the purpose of quality checking of a
clinical guideline.  

Previously, we have shown that the theory of abductive diagnosis can
be taken as a foundation for the formalisation of quality criteria of
a clinical guideline \cite{lucas2003} and that these can be verified
using (interactive) program verification techniques
\cite{HommersomIEEE2007}. In this paper, we provide an alternative to
this approach by translating this formalism, a restricted part of
temporal logic, to standard first-order logic.  Furthermore, the
quality criteria are interpreted in such a way that they can be stated
in terms of a monotonic entailment relation.  We show that, because of
the restricted language needed for the formalisation of the guideline
knowledge, the translation is a relatively simple fragment of
first-order logic which is amended to automated verification.  Thus,
we show that it is indeed possible, while not easy, to cover the route
from informal medical knowledge to a logical formalisation and
automated verification. 

The meta-level approach that is used here is particularly important
for the \emph{design} of clinical guidelines, because it corresponds
to a type of reasoning that occurs during the guideline development
process. Clearly, quality checks are useful during this process;
however, the design of a guideline can be seen as a very complex
process where formulation of knowledge and construction of conclusions
and corresponding recommendations are intermingled.  This makes it
cumbersome to do \emph{interactive} verification of hypotheses
concerning the optimal recommendation during the construction of such
a guideline, because guideline developers do not generally have the
necessary background in formal methods to construct such proofs
interactively. Automated theorem proving could therefore be
potentially more beneficial for supporting the guideline development
process.

The paper is organised as follows. In the next section, we start by
explaining what clinical guidelines are, and a method for formalising
guidelines by temporal logic is briefly reviewed. In Section
\ref{abductive} the formalisation of guideline quality using a
meta-level scheme that comes from the theory of abductive diagnosis
is described.  The guideline on the management of diabetes mellitus
type 2 that has been used in the case study is given attention in
Section \ref{guideline}, and a formalisation of this is given as well.
An approach to checking the quality of this guideline using the
reasoning machinery offered by automated reasoning tools is presented in
Section \ref{ar}.  Finally, Section \ref{conclusions} discusses what
has been achieved, and the advantages and limitations of this approach
are brought into perspective. In particular, we will discuss the role
of automated reasoning in quality checking guidelines in comparison to
alternative techniques such as model checking and interactive
verification.

\section{Framework}

In this section, we review the basics about clinical guidelines and
the temporal logic used in the remainder of the paper.

\subsection{Clinical Guidelines}

A clinical guideline is a structured document, containing detailed
advice on the management of a particular disorder or group of
disorders, aimed at health-care professionals. As modern guidelines
are based on scientific evidence, they contain information about the
quality of the evidence on which particular statements are based;
e.g., statements at the highest recommendation level are usually
obtained from randomised clinical trials \cite{woolf2000}.

The design of a clinical guideline is far from easy. Firstly, the
gathering and classification of the scientific evidence underlying and
justifying the recommendations mentioned in a guideline are time
consuming, and require considerable expertise of the medical field
concerned.  Secondly, clinical guidelines are very detailed, and making
sure that all the information contained in the guideline is complete
for the guideline's purpose, and based on sound medical principles is
hard work.

\begin{figure}[t]
\hrulefill
\begin{itemize}
\item Step 1: diet
\item Step 2: if Quetelet Index (QI) $\leq$ 27, prescribe a sulfonylurea
      drug; otherwise, prescribe a biguanide drug
\item Step 3: combine a sulfonylurea drug and biguanide (replace
      one of these by a $\alpha$-glucosidase inhibitor if side-effects
      occur)
\item Step 4: one of the following:
      \begin{itemize}
      \item oral anti-diabetics and insulin
      \item only insulin
      \end{itemize}
\end{itemize}
\vspace{-\baselineskip}
\hrulefill\hrulefill
\caption{Tiny fragment of a clinical guideline on the
management of diabetes mellitus type 2. If one of the steps $s$
where $s=1,2,3$ is ineffective, the management moves to step $s+1$.}
\label{db2man}
\end{figure}

An example of a part of a guideline is the
following (translated) text:
\begin{quote}
1. refer to a dietist; check blood glucose after 3 months \\
\indent
2. in case (1) fails and Quetelet Index (QI) $\leq$ 27, then administer a
sulfonylureum derivate (e.g. tolbutamide, 500 mg 1 time per day, max.
1000 mg 2 per day) and in case of Quetelet Index (QI) $>$ 27 biguanide (500 mg 1 per
day, max. 1000 mg 3 times per day); start with lowest dosage, increase
each 2-4 weeks if necessary
\end{quote}
It is part of a real-world guideline for general practitioners about
the treatment of diabetes mellitus type 2. Part of this description
includes details about dosage of drugs at specific time periods. As we
want to reason about the general structure of the guideline, rather
than about dosages or specific time periods, we have made an
abstraction as shown in Fig.~\ref{db2man}.  This guideline fragment
is used in this paper as a running example.

Guidelines can be as large as 100 pages; however, the number of
recommendations they include are typically few.  In complicated
diseases, each type of disease is typically described in different
sections of a guideline, which provides ways to modularise the
formalisation in a natural fashion.  For example, in the Protocure
project, we have formalised an extensive guideline about breast cancer
treatment, which includes recommendations very similar in nature and
structure to the abstraction shown in Fig.~\ref{db2man}. In this
sense, the fragment in Fig.~\ref{db2man} can be lookup upon as one
of the recommendations in any guideline whatever its size. Clinical
\emph{protocols} are normally more detailed, and the abstraction used
here will not be appropriate if one wishes to consider such details in
the verification process.  For example, in the Protocure project we
also carried out work on the verification of a clinical protocol about
the management of neonatal jaundice, where we focussed on the levels
of a substance in the blood (bilirubin) \cite{Teije:2006AIM}.
Clearly, in this case abstracting away from substance levels would be
inappropriate. 

The conclusions that can be reached by the rest of the paper are
relative to the abstraction that was chosen. The logical methods
that we employ are related to this level of abstraction, even
though other logical methods are
available to deal issues such as more detailed temporal
reasoning \cite{Moszkowski:1985Computer} or probabilities
\cite{Richardson:2006ML,Kersting:2000ILP}, which might be necessary for
some guidelines or protocols.  Nonetheless, where development of an
abstraction of a medical document will be necessary for any
verification task, the way it is done is dependent on what is being
verified and the nature of the document. 
The level of abstraction that we employ allow us to reason about the
structure and effects of treatments, which, in our view, is the most
important aspect of many guidelines. 

One way to use formal methods in the context of guidelines is to
automatically verify whether or not a clinical guideline fulfils particular
properties, such as whether it complies with quality \emph{indicators}
as proposed by health-care professionals \cite{marcos2002}. For
example, using particular patient assumptions such as that after
treatment the levels of a substance are dangerously high or low, it is
possible to check whether this situation does or does not violate the
guideline.  However, verifying the effects of treatment as well as
examining whether a developed clinical guideline complies with global
criteria, such as that it avoids the prescription of redundant drugs,
or the request of tests that are superfluous, is difficult to
impossible if only the guideline text is available.  Thus, the
capability to check whether a guideline fulfils particular medical
objectives may require the availability of more medical knowledge than
is actually specified in a clinical guideline.  How much additional
knowledge is required may vary from guideline to guideline.  In the
development of the theory below it is assumed that at least some
medical background knowledge is required; the extent and the purpose
of that background knowledge is subsequently established using the
diabetes mellitus type 2 guideline.  The development, logical
implementation, and evaluation of
a formal method that supports this process is the topic of the
remainder of the paper.

\subsection{Using Temporal Logic in Clinical Guidelines}
\label{temp.log}

\begin{table}
\caption{Used temporal operators; $t$ stands for a time instance.}
\label{temp}
\centering\begin{tabular}{lll}
\hline\hline
{\bf Notation} & {\bf Informal meaning} & {\bf Formal meaning} \\
\hline
$\textsf{H} \varphi$ & $\varphi$ has always been true in the past & $t
\vDash \textsf{H} \varphi$ iff $\forall t' < t: t' \vDash \varphi$ \\
$\varphi\,\textsf{U}\,\psi$ & $\varphi$ is true until $\psi$ holds &
$\vDash \varphi\,\textsf{U}\,\psi$ iff $\exists t' \geq t: t' \vDash \psi$ \\ 
& & $\textup{and } \forall t'': t \leq t'' < t' \rightarrow t'' \vDash \varphi$ \\
\hline\hline
\end{tabular}
\vspace{-\baselineskip}
\end{table}

Many representation languages for formalising and reasoning about
medical knowledge have been proposed, e.g., predicate logic
\cite{lucas93}, (heuristic) rule-based systems
\cite{Shortliffe:1974}, and causal representations \cite{Patil:1981}.
It is not uncommon to abstract from time in these representations;
however, as medical management is very much a time-oriented process,
guidelines should be looked upon in a temporal setting.  It has been
shown previously that the step-wise, possibly iterative, execution of
a guideline, such as the example in Fig.~\ref{db2man}, can be
described precisely by means of temporal logic \cite{Teije:2006AIM}.
In a more practical setting it is useful to support the modelling
process by means of tools.
There is promising research for maintaining a logical
knowledge base in the context of the semantic web (e.g., the
Prot\'eg\'e-OWL editor\footnote{\url{http://protege.stanford.edu/overview/protege-owl.html}
[Accessed: 21 May 2008]}),  
and the logical formalisation
described in this paper could profit from the availability of such
tools.

The temporal logic that we use here is a modal logic, where
relationships between worlds in the usual possible-world semantics of
modal logic is understood as time order, i.e., formulae are
interpreted in a \emph{temporal frame} ${\cal F} = ({\mathbb T},<,I)$,
where ${\mathbb T}$ is the set of intervals or time points, $<$ a time
ordering, and $I$ an interpretation of the language elements with
respect to ${\mathbb T}$ and $<$. The language of
first-order logic, with equality and unique names assumption, is
augmented with the operators $\textsf{U}$, $\textsf{H}$,
\textsf{G}, $\textsf{P}$, and $\textsf{F}$, where the
temporal semantics of the first two operators is defined in 
Table~\ref{temp}.  The last four operators are simply defined in terms of
the first two operators:
\[
\begin{array}{lr}
\vDash \textsf{P} \varphi \leftrightarrow \neg \textsf{H} \neg \varphi & \mbox{(somewhere in the past)} \\
\vDash \textsf{F} \varphi \leftrightarrow \top \textsf{U} \varphi
& \mbox{(now or somewhere in the future)} \\
\vDash \textsf{G} \varphi \leftrightarrow \neg \textsf{F} \neg\varphi &
\mbox{(now and always in the future)} 
\end{array}
\]
This logic offers the right abstraction level to cope with the nature
of the temporal knowledge in clinical guidelines required for our purposes.

Other modal operators added to the language of first-order logic
include $\textsf{X}$, where $\textsf{X} \varphi$ has the operational
meaning of an execution step, followed by execution of program part
$\varphi$. Even though this operator is not explicitly used in our 
formalisation of medical knowledge, a principle similar to the
semantics of this operator
is used in Section \ref{section:verifplans} for reasoning about the
step-wise execution of the guideline.

In addition, axioms can be added that indicate
that progression in time is \emph{linear} (there are other
possible axiomatisations, such as branching time, see \cite{turner85}).
The most important of these are:
\begin{itemize}
\item [(1)] \emph{Transitivity}: $\vDash \textsf{F}\textsf{F} \varphi \rightarrow \textsf{F} \varphi$
\item [(2)] \emph{Backward linearity}: $\vDash (\textsf{P} \varphi \wedge \textsf{P} \psi) \rightarrow (\textsf{P}(\varphi \wedge \psi) \vee \textsf{P}(\varphi \wedge \textsf{P}\psi) \vee \textsf{P}(\textsf{P}\varphi \wedge \psi))$
\item [(3)] \emph{Forward linearity}: $\vDash (\textsf{F} \varphi \wedge \textsf{F}
  \psi) \rightarrow (\textsf{F}(\varphi \wedge \psi) \vee \textsf{F}(\varphi \wedge
  \textsf{F}\psi) \vee \textsf{F}(\textsf{F}\varphi \wedge \psi))$
\end{itemize}
Transitivity ensures that we can move along the time axis from the
past into the future; backward and forward linearity ensure that the
time axis does not branch. Consider, for example, axiom (3), which
says that if there exists a time $t$ when $\varphi$ is true, and a
time $t'$ when $\psi$ holds, then there are three possibilities:
$\varphi$ and $\psi$ hold at the same time, or at some time in the
future $\varphi$ and further away in the future $\psi$ hold; the
meaning of the last disjunct is similar.  Other useful axioms concern
the boundedness of time; assuming that time has no beginning and no
end, gives rise to the following axioms: $\vDash \textsf{H} \varphi \rightarrow
\textsf{P} \varphi$ and $\vDash \textsf{G} \varphi \rightarrow \textsf{F} \varphi$.

Alternative formal languages for modelling medical knowledge are
possible. For example, differential equations describing compartmental
models that are used to predict changes in physiological variables in
individual patients have been shown to be useful (e.g.,
\cite{Magni:2000ABE,Lehmann:1998CMPB}). In the context of
clinical reasoning they are less useful, as they essentially concern
levels of substances as a function of time and, thus, do not offer the
right level of abstraction that we are after.

\section{Application to Medical Knowledge}
\label{abductive}

It is well-known that knowledge elicitation is difficult (see e.g.,
\cite{Evans:1988BIT}) and due to complexity and uncertainty this is
particularly true for medical knowledge (see e.g., \cite{HandbookMI}).
The effort to acquire this knowledge is dependent on the availability
of the knowledge in the guideline and the complexity of the mechanisms
that are involved in the development of the disease. For
evidence-based guidelines, a large part of the relevant knowledge
required for checking the quality of the recommendations is included
in the guideline, which makes the problem more contained than the
problem of arbitrary medical knowledge elicitation.

The purpose of a clinical guideline is to have a certain positive
effect on the health status of a patient to which the guideline
applies. To establish that this is indeed the case, knowledge
concerning the normal physiology and abnormal, disease-related
pathophysiology of a patient is required.  Some of this physiological
knowledge may be missing from the clinical guidelines; however, much
of this knowledge can be acquired from textbooks on medical
physiology, which reduces the amount of effort required to construct
such knowledge models. The latter approach was taken in this research.

It is assumed that two types of knowledge are involved in
detecting the violation of good medical practice:
\begin{itemize}
\item Knowledge concerning the (patho)physiological mechanisms
  underlying the disease, and the way treatment influences these
  mechanisms. The knowledge involved could be causal in nature, and is
  an example of \emph{object-knowledge}.
\item Knowledge concerning good practice in treatment selection; this
  is \emph{meta-knowledge}.
\end{itemize}
Below we present some ideas on how such knowledge may be formalised
using temporal logic (cf.\ \cite{lucas95} for earlier work in the area
of formal modelling of medical knowledge).

We are interested in the prescription of drugs, taking into account
their mode of action. Abstracting from the dynamics of their
pharmacokinetics, this can be formalised in logic as follows:
\begin{equation}
(\textsf{G}\,d \wedge r) \rightarrow \textsf{G}(m_1 \wedge \cdots \wedge m_n)
\end{equation}
\noindent
where $d$ is the name of a drug, $r$ is a (possibly negative or empty)
\emph{requirement} for the drug to take effect, and $m_k$ is a mode of
action, such as decrease of release of glucose from the liver, which
holds at all future times. Note that we assume that drugs are applied
for a long period of time, here formalised as `always'. This is
reasonable if we think of the models as finite structures that
describe a somewhat longer period of time, allowing the drugs to take
effect. Synergistic effects and interactions amongst drugs can also be
formalised along those lines, as required by the guideline under
consideration. This can be done either by combining their
joint mode of action, by replacing $d$ in the formula above by a
conjunction of drugs, by defining harmful joint effects of drugs in
terms of inconsistency, or by reasoning about modes of actions.  As we
do not require this feature for the clinical guideline considered in
this paper, we will not go into details. In addition, it is possible
to reason about such effects using special purpose temporal logics
with abstraction and constraints, such as developed by Allen
\cite{Allen:ACM1983} and Terenziani \cite{TerenzianiCI2000} without a
connection to a specific field, and by Shahar \cite{shahar1997Temp}
for the field of medicine.  Thus, temporal logics are expressive
enough to cope with extensions to the formalisation as used in this
paper.

The modes of action $m_k$ can be combined, together with an
\emph{intention} $n$ (achieving normoglycaemia, i.e., normal blood
glucose levels, for example), a particular patient \emph{condition}
$c$, and \emph{requirements} $r_j$ for the modes of action to be
effective:
\begin{equation}
(\textsf{G} m_{i_1} \wedge \cdots \wedge \textsf{G} m_{i_m}\wedge r_1 \wedge \cdots \wedge r_p \wedge  \textsf{H} c ) \rightarrow \textsf{G} n
\end{equation}
For example, if the mode describes that there is a stimulus to
secrete more insulin and the requirement that sufficient capacity to provide
this insulin is fulfilled, then the amount of glucose in the blood
will decrease.

Good practice medicine can then be formalised as follows.  Let ${\cal
  B}$ be background knowledge, $T \subseteq \{d_1,\ldots,d_p\}$ be a
set of drugs, $C$ a collection of patient conditions, $R$ a collection
of requirements and $N$ a collection of intentions which the
physician has to achieve. As an abbreviation, the union of $C$ and $R$, i.e., the
variables describing the patient, will be referred
to as $P$, i.e., $P = C \cup R$.
Finding an acceptable treatment given
such knowledge amounts to finding an explanation, in terms of a
treatment, that the intention will be achieved. Finding the best
possible explanation given a number of findings is called
\emph{abductive} reasoning \cite{console91,poole90}. We say that a 
set of drugs $T$ is a \emph{treatment} according to the theory of 
abductive reasoning if \cite{lucas2003}:
\begin{description}
\item [(M1)] ${\cal B} \cup \textsf{G} T \cup P \nvDash \bot$ (the drugs do not have contradictory effects), and
\item [(M2)] ${\cal B} \cup \textsf{G} T \cup P \vDash N$ (the drugs
  handle all the patient problems intended to be managed).
\end{description}
One could think of the formula ${\cal B} \cup \textsf{G} T \cup P$
as simulating a particular patient $P$ given a particular treatment
$T$. For each relevant patient groups, these
properties can be investigated.
If in addition to \textbf{(M1)} and \textbf{(M2)} condition
\begin{description}
\item [(M3)] $O_\varphi(T)$ holds, where $O_\varphi$ is a meta-predicate
  standing for an optimality criterion or combination of optimality
  criteria $\varphi$, then the treatment is said to be \emph{in
    accordance with good-practice medicine}.
\end{description}
A typical example of this is subset minimality $O_\subset$:
\begin{equation}
O_\subset(T) \equiv \forall T' \subset T: T' \; \mbox{is not a
treatment according to \textbf{(M1)} and \textbf{(M2)}}
\label{eq:subsetmin}
\end{equation}
\noindent
i.e., the minimum number of effective drugs are being prescribed. For
example, if $\{d_1,d_2,d_3\}$ is a treatment that satisfies condition
\textbf{(M3)} in addition to \textbf{(M1)} and \textbf{(M2)}, then the subsets $\{d_1,d_2\}$,
$\{d_2,d_3\}$, $\{d_1\}$, and so on, do not satisfy conditions 
\textbf{(M1)} and \textbf{(M2)}. 
In the context of abductive reasoning, subset minimality is often
used in order to distinguish between various solutions; it is also
referred to in literature as \emph{Occam's razor}. 
Another definition
of the meta-predicate $O_\varphi$ is in terms of minimal cost $O_c$:
\begin{equation}
O_c(T) \equiv \forall T', \mbox{with}\; T' \;\mbox{a treatment:}\;\; c(T') \geq c(T)
\end{equation}
\noindent
where $c(T) = \sum_{d \in T} \textit{cost}(d)$; combining the two
definitions also makes sense. For example, one could come up with a
definition of $O_{\subset,c}$ that among two subset-minimal treatments selects the
one that is the cheapest in financial or ethical sense.

The quality criteria that we have presented in this section could also
be taken as starting points for \emph{critiquing}, i.e., criticising
clinical actions performed and recorded by a physician
(cf.~\cite{Miller:1984} for an early critiquing system), especially if
we consider the formalisation of the background knowledge a model for
simulating a patient receiving a specific treatment. However, here we 
look for means to criticise the recommendations given by the
guidelines.

In order to verify the quality of guidelines, we do not make use of
data from medical records.  The use of such data is especially
important if one wishes to empirically evaluate the guideline. As data
may be missing from the database---a very common situation in clinical
datasets---tests ordered for a patient and treatments given to the
patient may not be according to the guideline. Therefore, such
datasets cannot be used to identify problems with the clinical
guideline. Results would tell as much about the dataset as about the
guidelines. Once the guideline has been shown to be without flaws, it
becomes interesting to carry out subsequent evaluation of the
guideline using patient data.  These were the main reasons why we
explored guideline quality by using well-understood and well-described
data from hypothetical patients; this simulates the way medical
doctors would normally critically look at a guideline. Notice the
similarity with use-cases in software engineering. This method is
practical and possible, and could be used in the process of designing
a guideline.

\section{Management of Diabetes Mellitus Type 2}
\label{guideline}

To determine the global quality of the guideline, the background
knowledge itself was only formalised so far as required for
investigating the usefulness of the theory of quality checking
introduced above. The knowledge that is presented here was acquired
with the help of a physician, though this knowledge can be found in
many standard textbooks on physiology (e.g.,
\cite{Ganong:2005,Guyton:2000}).

\subsection{Initial Analysis}

It is well known that diabetes type 2 is a very complicated disease:
various metabolic control mechanisms are deranged and many different
organ systems, such as the cardiovascular and renal system, may be
affected by the disorder.  Here we focus on the derangement of glucose
metabolism in diabetic patients, and even that is nontrivial.  To
support non-expert medical doctors in the management of this
complicated disease in patients, access to a guideline is really
essential.

One would expect that as this disorder is so complicated, the diabetes
mellitus type 2 guideline is also complicated. This, however, is not
the case, as may already be apparent from the guideline fragment shown
in Fig.~\ref{db2man}. This indicates that much of the
knowledge concerning diabetes mellitus type 2 is missing from the
guideline, and that without this background knowledge it will be
impossible to spot the sort of flaws we are after. Hence, the
conclusion is that a deeper analysis is required; the results of such
an analysis are discussed next.

\subsection{Diabetes Type 2 Background Knowledge}
\label{bgknowledge}
\begin{figure}
\centerline{\includegraphics[width=\textwidth]{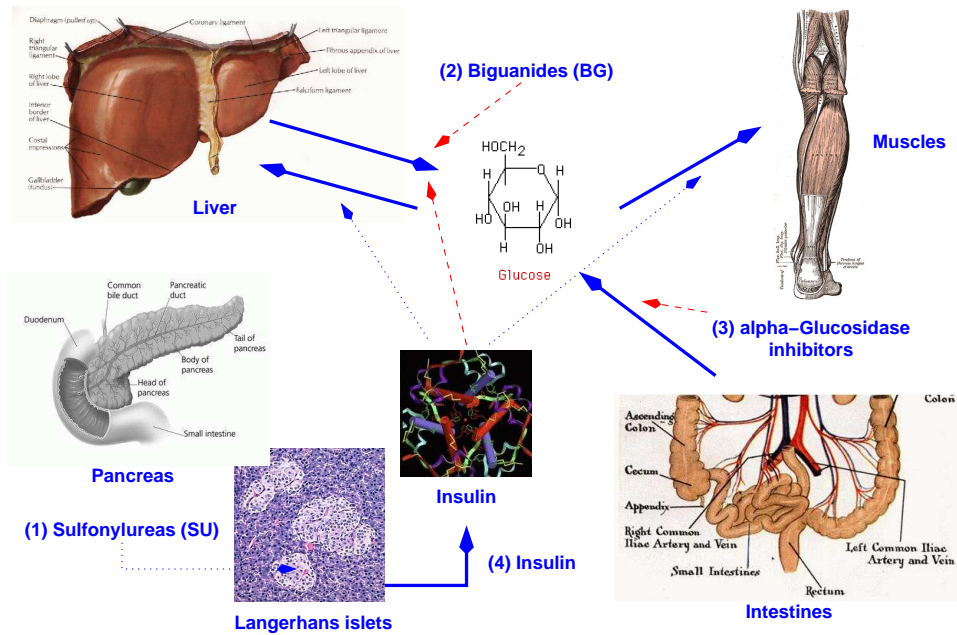}}
\caption{Summary of drugs and mechanisms controlling
  the blood level of glucose; $-$ $-$ $\rightarrow$: inhibition,
$\cdots\cdots$$\rightarrow$: stimulation.}
\label{diabetes}
\end{figure}

Fig.~\ref{diabetes} summarises the most important mechanisms and
drugs involved in the control of the blood level of glucose. The protein
hormone insulin, which is produced by the \emph{B~cells} in the
Langerhans islets of the \emph{pancreas}, has the following major
effects:
\begin{itemize}
\item it increases the uptake of glucose by the liver, where it is
  stored as glycogen, and inhibits the release of glucose from the
  liver;
\item it increases the uptake of glucose by insulin-dependent tissues,
  such as muscle and adipose tissue.
\end{itemize}
At some stage in the natural history of diabetes mellitus type 2, the
level of glucose in the blood is too high (hyperglycaemia) due to
decreased production of insulin by the B~cells. A popular hypothesis
explaining this phenomenon is that target cells have become insulin
resistant, which with a delay causes the production of insulin by the
B~cells to raise.  After some time, the B~cells become exhausted, and
they are no longer capable of meeting the demands for insulin. As a
consequence, hyperglycaemia develops.

Treatment of diabetes type 2 consists of:
\begin{itemize}
\item Use of \emph{sulfonylurea} (SU) drugs, such as tolbutamid. These
  drugs stimulate the B~cells in producing more insulin, and if the
  cells are not completely exhausted, the hyperglycaemia can thus be
  reverted to normoglycaemia (normal blood glucose levels).
\item Use of \emph{biguanides} (BG), such as metformin. These
  drugs inhibit the release of glucose from the liver.
\item Use of \emph{$\alpha$-glucosidase inhibitors}. These
  drugs inhibit (or delay) the absorption of glucose from the intestines.
\item Injection of \emph{insulin}. This is the ultimate, causal treatment.
\end{itemize}
As insulin is typically administered by injection, in contrast to the
other drugs which are normally taken orally, doctors prefer to delay
prescribing insulin as long as possible.  Thus, the treatment part of
the diabetes type 2 guideline mentions that one should start with
prescribing oral antidiabetics (SU or BG, cf.\ Fig.~\ref{db2man}).
Two of these can also be combined if taking only one has insufficient
glucose-level lowering effect.  If treatment is still unsatisfactory,
the guideline suggests to: (1) either add insulin, or (2) stop with
the oral antidiabetics entirely and to start with insulin.

From a medical point of view, advice (1) above is somewhat curious.
If the oral antidiabetics are no longer effective enough, the B~cells
could be completely exhausted. Under these circumstances, it does not
make a lot of sense to prescribe an SU drug. The guideline here
assumes that the B~cells are always somewhat active, which may limit
the amount of insulin that has to be prescribed. Similarly,
prescription of a BG (or a $\alpha$-glucosidase inhibitor) is
justified, as by adding such an oral antidiabetic to insulin, the
number of necessary injections can be reduced from twice a day to once
a day. It should be noted that, when on insulin treatment, patients
run the risk of getting hypoglycaemia, which is a side
effect of insulin treatment not mentioned explicitly in the guideline.

The background knowledge concerning the (patho-)physiology of the
glucose metabolism as described above is formalised using temporal
logic, and kept as simple as possible. The specification is denoted by
${\cal B}_{\mbox{\scriptsize DM2}}$:
\begin{quote}
\begin{tabbing}
\textup{(1) }\=$\textsf{G}\,\textrm{Drug}(\textit{insulin}) \rightarrow \textsf{G}($\=$\textit{uptake}(\textit{liver},\textit{glucose}) = \textit{up} \; \wedge$ \\
 \>\> $\textit{uptake}(\textit{peripheral-tissues},\textit{glucose}) =
\textit{up})$\\[1ex]
\textup{(2)} $\textsf{G}(\textit{uptake}(\textit{liver}, \textit{glucose}) = \textit{up} \rightarrow
\textit{release}(\textit{liver}, \textit{glucose}) = \textit{down})$ \\[1ex]
\textup{(3)} $(\textsf{G}\,\textrm{Drug}(\textrm{SU}) \wedge \neg
\textit{capacity}(\textit{b-cells},
\textit{insulin}) = \textit{exhausted})$ \\
\> $\rightarrow \textsf{G}\textit{secretion}(\textit{b-cells}, \textit{insulin}) = \textit{up}$ \\[1ex]
\textup{(4)} $\textsf{G}\,\textrm{Drug}(\textrm{BG}) \rightarrow \textsf{G}\textit{release}(\textit{liver},
\textit{glucose}) = \textit{down}$
\\[2ex]
\textup{(5) }\=$($\=$\textsf{G}\textit{secretion}(\textit{b-cells}, \textit{insulin}) = \textit{up} \; \wedge$ 
 $\textit{capacity}(\textit{b-cells},\textit{insulin}) = \textit{subnormal} \; \wedge$ \\
 \>\> $\textrm{QI} \leq 27 \wedge \textsf{H}\, \textrm{Condition}(\textit{hyperglycaemia}))$ 
$\rightarrow \textsf{G}\, \textrm{Condition}(\textit{normoglycaemia})$ \\[2ex]
\textup{(6)} $(\textsf{G}\textit{release}(\textit{liver}, \textit{glucose}) = \textit{down} \; \wedge$ 
 $\textit{capacity}(\textit{b-cells}, \textit{insulin}) = \textit{subnormal} \; \wedge$ \\
\>\>$\textrm{QI} > 27 \wedge \textsf{H}\, \textrm{Condition}(\textit{hyperglycaemia}))$ 
$\rightarrow \textsf{G}\, \textrm{Condition}(\textit{normoglycaemia})$ \\[1ex]
\textup{(7) }\=$($\=$($\=$\textsf{G} \textit{release}(\textit{liver}, \textit{glucose}) = \textit{down} \; \vee$ 
  $\textsf{G} \textit{uptake}(\textit{peripheral-tissues}, \textit{glucose}) = \textit{up}) \; \wedge$ \\
 \>\>$\textit{capacity}(\textit{b-cells}, \textit{insulin}) = \textit{nearly-exhausted} \; \wedge$ 
 $\textsf{G}\textit{secretion}(\textit{b-cells}, \textit{insulin}) = \textit{up} \; \wedge$ \\
 \>\>$\textsf{H}\, \textrm{Condition}(\textit{hyperglycaemia}))$ 
$\rightarrow \textsf{G}\, \textrm{Condition}(\textit{normoglycaemia})$ \\[1ex]
\textup{(8)} $(\textsf{G}\textit{uptake}(\textit{liver}, \textit{glucose}) = \textit{up} \; \wedge$ 
 $\textsf{G}\textit{uptake}(\textit{peripheral-tissues}, \textit{glucose}) = \textit{up}) \; \wedge$ \\
 \>\>$\textit{capacity}(\textit{b-cells}, \textit{insulin}) = \textit{exhausted} \; \wedge$ 
 $\textsf{H}\, \textrm{Condition}(\textit{hyperglycaemia}))$ \\
\>$\rightarrow \textsf{G}(\textrm{Condition}(\textit{normoglycaemia}) \vee \textrm{Condition}(\textit{hypoglycaemia}))$ \\[1ex]
\textup{(9) }$(\textrm{Condition}(\textit{normoglycaemia})$ $\oplus$ $\textrm{Condition}(\textit{hypoglycaemia})$ $\oplus$ \\
\>\>$\textrm{Condition}(\textit{hyperglycaemia}))$ \= $\land$ 
$\neg (\textrm{Condition}(\textit{normoglycaemia})$ \= $\land$\\
\>\>$\textrm{Condition}(\textit{hypoglycaemia})$
$\land$ $\textrm{Condition}(\textit{hyperglycaemia}))$
\end{tabbing}
\end{quote}
where $\oplus$ stands for the exclusive OR. Note that when the B~cells
are exhausted, increased uptake of glucose by the tissues may result 
not only in normoglycaemia but also in hypoglycaemia. Note that this
background knowledge was originally developed for reasoning about the
application of a single treatment. It requires some modification in
order to reason about the whole guideline fragment (see Section
\ref{section:verifplans}).

\subsection{Quality Check}
\label{section:qualitycheck}

The consequences of various treatment options can be examined using the
method introduced in Section \ref{abductive}. Hypothetical patients
for whom it is the intention to reach a normal level of glucose in the
blood (normoglycaemia) and one of the steps
in the guideline is applicable in the guideline fragment given
in Fig.~\ref{db2man}, are considered, for example:
\begin{itemize}
\item Consider a patient with hyperglycaemia due to nearly exhausted
B~cells. For these patients, the third step of Fig.~\ref{db2man}
is applicable, so we check that:
\end{itemize}
\[
\begin{array}{l}
{\cal B}_{\mbox{\scriptsize DM2}} \cup \textsf{G}\,T \cup
\{\textit{capacity}(\textit{b-cells},\textit{insulin}) =
\textit{nearly-exhausted}\} \; \cup \\[1ex]
\;\;\;\;\;\{\textsf{H}\,
\textrm{Condition}(\textit{hyperglycaemia})\} \vDash \textsf{G}\, \textrm{Condition}(\textit{normoglycaemia})
\end{array}
\]
holds for $T = \{\textrm{Drug}(\textrm{SU}),\textrm{Drug}(\textrm{BG})\}$,
which also satisfies the minimality condition $O_\subset(T)$.
\begin{itemize}
\item Prescription of treatment $T =
  \{\textrm{Drug}(\textrm{SU}),\textrm{Drug}(\textrm{BG}),\textrm{Drug}(\textrm{insulin})\}$
  for a patient with exhausted B~cells, for which the intended
  treatment regime is described in the fourth step of Fig.~\ref{db2man}, yields:
\end{itemize}
\[
\begin{array}{l}
{\cal B}_{\mbox{\scriptsize DM2}} \cup \textsf{G}\, T \cup
  \{\textit{capacity}(\textit{b-cells},\textit{insulin}) =
  \textit{exhausted}\} \; \cup \\[1ex]
  \;\;\;\;\{\textsf{H}\,\textrm{Condition}(\textit{hyperglycaemia})\}
  \vDash \\[1ex]
\;\;\;\;\;\;\;\;\textsf{G}(\textrm{Condition}(\textit{normoglycaemia}) \vee\textrm{Condition}(\textit{hypoglycaemia}))
\end{array}
\]
In the last case, it appears that it is possible that a patient
develops hypoglycaemia due to treatment; if this possibility is
excluded from axiom (8) in the background knowledge, then the 
minimality condition $O_\subset(T)$, and also
$O_{c}(T)$, does not hold since insulin by itself is enough to
reach normoglycaemia. In either case, good practice medicine is
violated, which is to prescribe as few drugs as possible, taking into
account costs and side-effects of drugs.  Here, three drugs are
prescribed whereas only two should have been prescribed (BG and
insulin, assuming that insulin alone is too costly), and the possible
occurrence of hypoglycaemia should have been prevented.

\section{Automated Quality Checking}
\label{ar}

\begin{figure}
\centerline{\includegraphics[width=250pt]{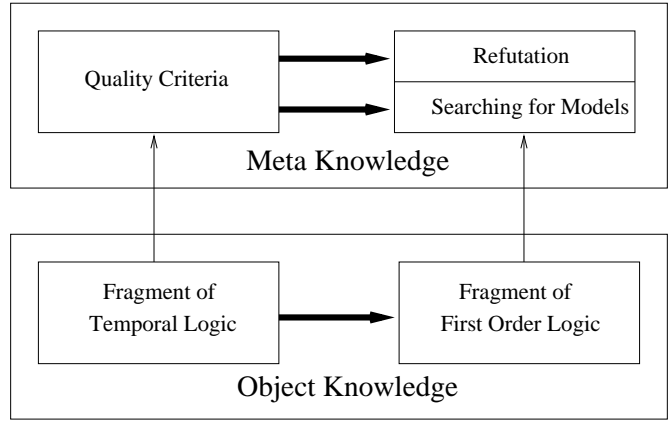}}
\caption{Translation of medical knowledge.}
\label{translation.outline}
\end{figure}

As mentioned in the introduction, we have explored the feasibility of using
automated reasoning tools to
check the quality of guidelines, in the sense described above. 

\subsection{Motivation for using Automated Reasoning}

Several techniques are available for reasoning with temporal logic.
Firstly, an automated theorem prover aims at proving theorems without
any interaction from the user. This is a problem with high complexity;
e.g., for first-order logic, this problem is recursively enumerable.
For this reason, \emph{interactive theorem proving} has been used as
an alternative, where it is possible and sometimes necessary to give
hints to the system. As a consequence, more complicated problems can
be handled; however, in the worst case every step of the proof has to
be performed manually.

For our work, it is of interest to obtain insight how much of the
proof effort can be automated as this would clearly improve the
practical usefulness of employing formal methods in the process of
guideline development. In our previous work we have considered using
interactive theorem proving \cite{HommersomIEEE2007}.  This was a
successful experiment; however, the number of interactions that were
needed were still high and a lot of expertise in the area of theorem
proving is required for carrying out this task. Furthermore, there has
been considerable progress in terms of speed and the size of problems
that theorem provers can handle \cite{casc}.  In our opinion, these
developments provide enough justification to explore the use of
automated reasoning techniques in combination with specific
strategies.

One of the most important application areas of model finders and
theorem provers is program verification. In programs
there is a clear beginning of the execution, which makes it intuitive
to think about properties that occur after the start of the program.
Therefore, it is not surprising that much work that has been done in the context
of model finding and theorem proving only deals with the future time
modality. However, it is more natural to model medical knowledge with
past time operators, i.e., what happened to the patient in the past.
It is well-known that formulas with a past-time modality can be mapped
to a logical formula with only future time modalities such that both
formulas are equivalent for some initial state \cite{gab87}. 
The main drawback of this approach is that formulas will get
much larger in size \cite{markey2003} and as a consequence become much
harder to verify in a theorem prover designed for modal logics. 

For this reason, we have chosen to use an alternative approach which
uses a \emph{relational translation} to map the temporal logic
formulas to first-order logic. As primary tools we used the
resolution-based theorem prover \otter{} \cite{mccune03} and the
finite model searcher \mace{} \cite{mccune01}, which take first-order
logic with equality as their input.  These systems have been optimised
for reasoning with first-order logical formulas and offer various
reasoning strategies to do this efficiently. For example, \otter{}
offers the set-of-support strategy and hyperresolution as efficient
reasoning methods. There are alternative systems that could have
been used;
however, it is not the aim of this paper to compare these systems.
\otter{} has been proven to be robust and efficient, and has
been successfully applied to solve problems of high complexity,
for example in the area of algebra \cite{Philips:2005JSC} and
logic \cite{Jech:1995JAR}.

There has been work done to improve the speed of resolution-based
theorem provers on modal formulas \cite{areces2000}, but again,
converse modalities such as the past-time operators are not
considered.  We found that the general heuristics
applicable to full first-order logic are sufficient to our task.

\subsection{Translation}
\label{translation}

In order to prove meta-level properties, it is necessary to reason at
the object-level. Object-level properties typically do not contain
background knowledge concerning the validity what it being verified.
For example, the (M2) property of Section \ref{abductive} has a clear
meaning in terms of clinical guidelines, which would be lost if stated
as an object-level property. Moreover, it is not (directly) possible
to state that something does \emph{not} follow at the object level.
Fig.~\ref{translation.outline} summarises the general approach. We
will first give a definition for translating the object knowledge to
standard logic and then the translation of the meta-level knowledge
will follow.

\subsubsection{Translation of Object Knowledge}
\label{trans.object}

The background knowledge, as defined in Subsection \ref{bgknowledge},
is translated into first order logic.  For every function
$f$ with two elements in the co-domain, call these $\{ c_1, c_2\}$, we
introduce a fresh variable $p$ for every element $a$ in the domain
such that $f(a) = c_1$ holds iff $p$ holds, and $f(a) = c_2$ holds iff
$\neg p$ holds.  
For example, axiom (2) of $\mathcal{B}_{DM2}$ in Section
\ref{bgknowledge} is represented by defining 
`$\textit{uptake}(\textit{liver}, \textit{glucose}) = \textit{up}$' and 
`$\textit{release}(\textit{liver}, \textit{glucose}) = \textit{up})$' as
propositions and rewriting this axiom as:
\[
\textsf{G}(`\textit{uptake}(\textit{liver}, \textit{glucose}) =
\textit{up}\mbox{'} \rightarrow
\neg (`\textit{release}(\textit{liver}, \textit{glucose}) =
\textit{up}\mbox{'}))
\]
For
the \textit{capacity} function, a function with three elements in its
co-domain, we add a proposition $p_x$ for each expression
$\textit{capacity}(\textit{b-cells}, \textit{insulin}) = x$ and an
axiom saying that each pair of these propositions are mutually
exclusive. Finally, the term $\textrm{QI} > 27$ is interpreted as a
proposition as well, i.e., we do not reason about the numerical
value of QI.

Technically, this translation is not required, since we could extend
the translation below to full first-order temporal logic. In practice
however, we would like to avoid additional complexity from first-order
formulas during the automated reasoning.

The relational translation (e.g.,
\cite{moore80,areces2000,schmidt2003}) $\ST_{t}(\varphi)$, also
referred to as the \emph{standard translation}, translates a propositional
temporal logical formula $\varphi$ into a formula in a first-order
logic with (time-indexed) unary predicate symbols $P$ for every
propositional variable $p$ and one binary predicate $>$.  It is
defined as follows, where $t$ is an individual variable standing for
time:
\[
\begin{array}{lll}
\ST_{t}(p) & \textup{iff} & P(t) \\
\ST_{t}(\neg \varphi) & \textup{iff} & \neg \ST_{t}(\varphi) \\
\ST_{t}(\varphi \land \psi) & \textup{iff} & \ST_{t}(\varphi) \land \ST_{t}(\psi) \\
\ST_{t}(\textsf{G} \varphi) & \textup{iff} & \forall t'\; (t \not>  t' \rightarrow \ST_{t'}(\varphi)) \\
\ST_{t}(\textsf{H} \varphi) & \textup{iff} & \forall t'\; (t > t' \rightarrow \ST_{t'}(\varphi)) \\
\end{array}
\]
Note that the
last two elements of the definition give the meaning of the
$\textsf{G}$ modality and its converse, the $\textsf{H}$ modality. For
example, the formula $\textsf{G} (p \rightarrow \textsf{P} p)$
translates to $\forall t_2\; (t \not>  t_2 \rightarrow (P(t_2)
\rightarrow \exists t_3\; (t_2 > t_3 \land P(t_3)))$. It is
straightforward to
show that a formula in temporal logic is satisfiable if and only if
its relational translation is. Also, recall that we use set union to denote conjunction, thus
$\ST_{t} (\Gamma \cup \Delta)$ is
defined as $\ST_{t} (\Gamma) \land \ST_{t} (\Delta)$.  

In the literature a functional approach to translating modal logic has
appeared as well \cite{ohlbach1988}, which relies on a non-standard
interpretation of modal logic and could be taken as an alternative to
this translation.

\subsubsection{Translation of Meta-level Knowledge}
\label{trans.meta}

Again, we consider the criteria for good practice medicine and make
them suitable for use with the automated reasoning tools. In order to stress
that we deal with provability in these tools, we use the `$\vdash$' 
symbol instead of the `$\models$' (validity) symbol. We say that a
treatment $T$ is a treatment complying with the requirements of good
practice medicine iff:
\begin{description}
\item [(M1$'$)] $\ST_{t}(\mathcal{B}\cup\textsf{G} T \cup C \cup R)
\nvdash \bot$
\item [(M2$'$)] $\ST_{t}(\mathcal{B}\cup\textsf{G} T \cup C \cup R \cup
\neg N) \vdash \bot$
\item [(M3$'$)] $\forall T' \subset T: T' \; \mbox{is not a treatment
according to \textbf{(M1$'$)} and \textbf{(M2$'$)}}$
\noindent
\end{description}
Criterion \textbf{(M3$'$)} is a specific instance of \textbf{(M3)},
i.e., subset minimality as explained in Section \ref{abductive}
(Equation (\ref{eq:subsetmin})). As the relational translation preserves
satisfiability, these quality requirements are equivalent to their
unprimed counterparts in Section \ref{abductive}. To automate this reasoning process we use
\mace{} to verify \textbf{(M1$'$)}, \otter{} to verify
\textbf{(M2$'$)}, and \textbf{(M3$'$)} can be seen as a combination of
both for all subsets of the given treatment.

\subsection{Results}
\label{implementation}

In this subsection we will discuss the actual implementation in
\otter{} and some results obtained by using particular heuristics.

\subsubsection{Resolution Strategies}
\label{strategies}

An advantage that one gains from using a standard theorem prover 
that a whole range of different resolution rules and search
strategies are available and can be varied depending on the 
problem. \otter{} uses the set-of-support strategy
\cite{wos65} as a standard strategy. In this strategy the original
set of clauses is divided into a \emph{set-of-support} and a
\emph{usable} set such that in every resolution step at least one of
the parent clauses has to be member of the set-of-support and each
resulting resolvent is added to the set-of-support.

Looking at the structure of the formulas in Section \ref{guideline},
one can see that formulas are of the form $p_0\wedge \cdots \wedge p_n
\rightarrow q$, where $p_0 \wedge \cdots \wedge p_n$ and $q$ are
almost all
positive literals. Hence, we expect the occurrence of mainly negative literals in our
clauses, which can be exploited by using negative hyperresolution
(neg\_hyper for short) \cite{robinson65} in
\otter. With this strategy a clause with at least one positive literal
is resolved with one or more clauses only containing negative
literals (i.e., negative clauses), provided that the resolvent is a
negative clause. The parent clause with at least one positive literal
is called the \emph{nucleus}, and the other, negative, clauses are
referred to as the \emph{satellites}. 

\subsubsection{Verification of Treatments}
\label{proofs}

The ordering predicate $>$ that was introduced in Section
\ref{trans.object} was defined by adding axioms of irreflexivity, anti-symmetry, and
transitivity. We did not find any cases where the axiom of
transitivity was required to construct the proof, which can be
explained by the low modal depth of our formulas.  As a consequence,
the axiom was omitted with the aim to improve the speed of theorem
proving. Furthermore, because we lack the next step modality, we did not
need to axiomatise a subsequent time point. Experiments showed that
this greatly reduces the amount of effort for the theorem prover.

We used \otter{} to perform the two proofs which are instantiations of
\textbf{(M2$'$)}.  First we, again, consider a patient with
hyperglycaemia due to nearly exhausted B~cells and prove:
\[
\begin{array}{l}
\ST_{0}({\cal B}_{\mbox{\scriptsize DM2}} \cup \textsf{G}\,T
\cup \{\textit{capacity}(\textit{b-cells},\textit{insulin}) =
\textit{nearly-exhausted}\} \\[1ex]
\;\;\;\;\;\cup \; \{ \textsf{H}\,
\textrm{Condition}(\textit{hyperglycaemia})\} \\[1ex] 
\;\;\;\;\;\cup \; \{ \neg
\textsf{G}\,\textrm{Condition}(\textit{normoglycaemia}) \})
\vdash \bot
\end{array}
\]
where $T = \{\textrm{Drug}(\textrm{SU}),\textrm{Drug}(\textrm{BG})\}$,
i.e., step 3 of the guideline (see Fig.~\ref{db2man}).
Note that we use `$0$' here to represent the current time point.
This property was proven using \otter{} in 62 resolution steps with the
use of the negative hyperresolution strategy.
A summary of this proof can be found in Appendix \ref{firstproof}.

Similarly, given $T =
\{\textrm{Drug}(\textrm{SU}),\textrm{Drug}(\textrm{BG}),\textrm{Drug}(\textrm{insulin})\}$
to a patient with exhausted B~cells, as suggested by the
guideline in step 4, it follows that:
\[
\begin{array}{l}
{\ST_{0}(\cal B}_{\mbox{\scriptsize DM2}} \cup \textsf{G}\, T \cup
  \{\textit{capacity}(\textit{b-cells},\textit{insulin}) =
  \textit{exhausted}\} \; \cup \\[1ex]
\;\;\;\;\{\textsf{H}
  \textrm{Condition}(\textit{hyperglycaemia})\}\; \cup \\[1ex]
\;\;\;\;\{\neg
  (\textsf{G}(\textrm{Condition}(\textit{normoglycaemia})
\vee\textrm{Condition}(\textit{hypoglycaemia}))) \}) \vdash \bot
\end{array}
\]
However, if we take $T = \{ \textrm{Drug}(\textrm{insulin} \}$, the
same holds, which shows that, as already mentioned in
Section~\ref{section:qualitycheck}, that even if we ignore the fact
that the patient may develop \textit{hypoglycaemia}, the treatment is not minimal.
Compared to the previous property, a similar magnitude of
complexity in the proof was observed, i.e., 52 resolution steps.

\subsubsection{Using Weighting}
\label{weighing}

\begin{figure}
\begin{center}
\begin{tabular}{rrr}
\hline\hline
\textbf{Weights} & \textbf{Clauses (binary res)} & \textbf{Clauses (negative hyper res)} \\
\hline
$(0,1)$ & 17729 & 6994 \\
$(1,0)$ & 13255 & 6805 \\
$(1,1)$ & 39444 & 7001 \\
$(1,-1)$ & 13907 & 6836 \\
$(2,-2)$ & 40548 & 7001 \\
$(2,-3)$ & 16606 & 6805 \\
$(3,-4)$ & 40356 & 7095 \\
$(3,-5)$ & 27478 & 7001 \\
\hline\hline
\end{tabular}
\end{center}
\vspace{-\baselineskip}
\caption{Generated clauses to prove an instance of property
\textbf{M2$'$} depending on weights $(x,y)$ for the ordering relation on time.}
\label{weighting}
\end{figure}

\newcommand{\lorp}{\,\lor\,}

One possibility to improve the performance is by using term ordering
strategies. This will be explained below, but first we give a
motivating example why this is particularly useful for this class of
problems. Consider the following example
taken from \cite{areces2000}. Suppose we have the formula $\textsf{G} (p
\rightarrow \textsf{F} p)$. Proving this satisfiable amounts to
proving that the following two clauses are satisfiable:
\begin{enumerate}
\item $0 > t_1 \lorp \neg P(t_1) \lorp t_1 \not> f(t_1)$
\item $0 > t_2 \lorp \neg P(t_2) \lorp P(f(t_2))$
\end{enumerate}

It can be observed, that although we have two possibilities
to resolve these two clauses, for example on the $P$ literal, this is
useless because the negative $P$ literal is only bound by the
\textsf{G}-operator while the positive $P$ literal comes from a
formula at a deeper modal depth under the \textsf{F}-operator.
Suppose we resolve these $\neg P(t_1)$ and $P(f(t_2))$ and rename
$t_2$ to $t$, which generates the clause:
\[0 > f(t) \lorp f(t) \not> f(f(t)) \lorp 0 > t \lorp \neg
P(t) \]
and with (2) again we have:
\[0 > f(f(t)) \lorp f(f(t)) \not> f(f(f(t))) \lorp 0 > f(t)
\lorp c > t \lorp \neg P(t) \]
etc.
In this way, we can generate many new increasingly lengthy clauses.
Clearly,
these nestings of the Skolem functions will not help to find a a
contradiction more quickly if the depth of the modalities in the
formulas that we have is small, as the
new clauses are similar to previous clauses, except that they describe
a more complex temporal structure.

In \otter{} the weight of the clauses determines which clauses are
chosen from the set-of-support and usable list to become parents in a
resolution step. In case the weight of two clauses is the same, there
is a syntactical ordering to determine which clause has precedence.
This is called the Knuth-Bendix Ordering (KBO) \cite{kb1970}. As 
the goal of resolution is to find an empty clause, lighter
clauses are preferred. By default, the weight of a clause is the sum
of all occurring symbols (i.e., all symbols have weight 1) in the
literals. As we
have argued, since the temporal structure of our background knowledge
is relatively simple, nesting Skolem functions will not help to find
such an empty clause. Therefore it can be of use to manually change
the weight of the ordering symbol, which is done in \otter{} by a
tuple $(x,y)$ for each predicate, where $x$ is multiplied by the sum of
the weight of its arguments and is added to $y$ to calculate the new
weight of this predicate. For example, if $x = 2$ and $y = -3$, then
$v > w$ has a total weight of $2+2-3 = 1$, and $f(f(c)) > f(d)$ has
a weight of $2*3 + 2*2 - 3 = 7$.  

See Fig.~\ref{weighting} where we
show results when we applied this for some small values for $x$ and
$y$ for both binary and negative hyperresolution.  What these numbers
show (similar results were obtained for the other property) is
that the total weight of the ordering predicate should be smaller than
the weight of other, unary, predicates.  Possibly
somewhat surprisingly, the factor $x$ should not be increased too
much. Furthermore, in the case of a negative hyperresolution strategy the effect
is minimal. 

\subsection{Disproofs}
\label{disproofs}

\mace{} (Models And CounterExamples) is a program that
searches for small finite models of first-order statements using a
Davis-Putman-Loveland-Logemann decision procedure
\cite{davis60,davis62} as its core. Because of the relative simplicity
of our temporal formulas, it is to be expected that counterexamples
can be found rapidly, exploring only few states. Hence, it could be
expected that models are of the same magnitude of complexity as in the
propositional case and this was indeed the case. In fact, the
countermodels that \mace{} found consist of only 2 elements in the
domain of the model.

The first property we checked corresponded to checking whether the
background knowledge augmented with patient data and a therapy was
consistent, i.e., criterion \textbf{(M1$'$)}. Consider a
patient with hyperglycaemia due to nearly exhausted B~cells. We 
used \mace{} to verify:
\[
\begin{array}{l}
\ST_{0}({\cal B}_{\mbox{\scriptsize DM2}} \cup \textsf{G}\ T \cup
  \{\textit{capacity}(\textit{b-cells},\textit{insulin}) =
  \textit{exhausted}\} \; \cup \\[1ex]
\;\;\;\;\{\textsf{H}\,
  \textrm{Condition}(\textit{hyperglycaemia})\}) \nvdash \bot
\end{array}
\]
for  $T =
\{\textrm{Drug}(\textrm{SU}),\textrm{Drug}(\textrm{BG}),\textrm{Drug}(\textrm{insulin})\}$.
From this it follows that there is a model if $T =
\{\textrm{Drug}(\textrm{SU}),\textrm{Drug}(\textrm{BG})\}$ and
consequently we have verified \textbf{(M1$'$)}.
\begin{figure}
\begin{tabular}{l}
\hline\hline
\end{tabular}
\begin{verbatim}
     > :                   Condition(hyperglycaemia) :
         t | 0 1                  t 0 1
         ---+----                -------
         0 | F T                    T T
         1 | F F

   Drug(SU):                   Drug(BG) :
      t 0 1                       t  0 1
    -------                       -------
        F F                          T F

   capacity(b-cells, insulin) = nearly-exhausted :
         t | 0 1
         -------
             T T
\end{verbatim}
\begin{tabular}{l}
\hline\hline
\end{tabular}
\caption{Snippet from a \mace{} generated model. It lists the truth
value of all the unary predicates given each element of the domain
(i.e., the time points `\texttt{0}' and `\texttt{1}')
and every combination of domain elements for the binary predicate $<$.
Truth values are denoted by \texttt{T} (true) and \texttt{F} (false).}
\label{macemodel}
\end{figure}

Similarly, we found that for all $T \subset 
\{\textrm{Drug}(\textrm{SU}),\textrm{Drug}(\textrm{BG})\}$, it holds
that:
\[
\begin{array}{l}
\ST_{0}({\cal B}_{\mbox{\scriptsize DM2}} \cup \textsf{G}\,T 
\cup \{\textit{capacity}(\textit{b-cells},\textit{insulin}) 
= \textit{nearly-exhausted}\} \\[1ex]
\;\;\;\;\;\cup \; \{\textsf{H}\,
\textrm{Condition}(\textit{hyperglycaemia})\} \\[1ex] 
\;\;\;\;\;\cup \; \{ \neg
\textsf{G}\,\textrm{Condition}(\textit{normoglycaemia})\}) \nvdash \bot
\end{array}
\]
i.e., it is consistent to believe the patient will not have
normoglycaemia if less drugs are applied, which violates
\textbf{(M2)} for these subsets.
So indeed the conclusion was that the treatment complies with
\textbf{(M3$'$)} and thus complies with the criteria of good practice
medicine.  See for example Fig.~\ref{macemodel}, which contains a
small sample of the output that \mace{} generated. The output consists
of a
first-order model with two elements in the domain, named `0' and `1',
and an interpretation of all predicates and functions in this domain.
It shows that it is consistent with the background knowledge to
believe that the patient will continue to suffer from hyperglycaemia
if one of the drugs is not applied. Note that the model specifies that
biguanide is applied at the first time instance, as this is not
prohibited by the assumptions.

Finally, consider the treatment $T =
  \{\textrm{Drug}(\textrm{SU}),\textrm{Drug}(\textrm{BG}),\textrm{Drug}(\textrm{insulin})\}$
  for a patient with exhausted B~cells, we can show that:
\[
\begin{array}{l}
\ST_{0}({\cal B}_{\mbox{\scriptsize DM2}} \cup \textsf{G}\, T \cup
  \{\textit{capacity}(\textit{b-cells},\textit{insulin}) =
  \textit{exhausted}\} \; \cup \\[1ex]
  \;\;\;\;\{\textsf{H}
  \textrm{Condition}(\textit{hyperglycaemia})\}\;\cup \\[1ex]
\;\;\;\;
\{\textsf{G}(\textrm{Condition}(\textit{normoglycaemia})))\})
\nvdash \bot
\end{array}
\]
so the patient may be cured with insulin treatment, even though
this is not guaranteed as $\textrm{Condition}(\textit{normoglycaemia})$
does not deductively follow from the premises.
However, it is possible to prove the same property when $T =
\{\textrm{Drug}(\textrm{insulin})\}$ and thus \textbf{(M3$'$)} does
not hold in this case and as a consequence the guideline does not
comply with the quality requirements as discussed in Section
\ref{section:qualitycheck}.

\subsection{Plan Structure}
\label{section:verifplans}

So far, we have not considered the order in which treatments are being
considered and executed.
In this subsection, we look at the problem of reasoning about the order of
treatments described in the treatment plan listed in Fig.~\ref{db2man}.

\subsubsection{Formalisation}
\label{section:form}

In order to reason about a sequence of treatments, additional
formalisation is required. The background knowledge was developed for
reasoning about an individual treatment, and therefore, is parameterised
for the treatment that is being applied. We postulate
$\mathcal{B}_{\mbox{\scriptsize DM2}}$, parameterised by $s$, where $s$ is a
certain step in the protocol, i.e., $s = 1, 2, 3, 4$ (cf.
Fig.~\ref{db2man}; for example $s = 1$
corresponds to diet). The first axiom is then described by:
\[
\forall s\;(\textsf{G}\,\textrm{Drug}(\textit{insulin},s) \rightarrow
\textsf{G}(\textit{uptake}(\textit{liver},\textit{glucose},s) =
\textit{up}))
\]
The complete description of this background knowledge is 
denoted by $\mathcal{B}'_{\mbox{\scriptsize DM2}}$
The reason for this is that the `$\textsf{G}$' modality ranges over 
the time period of an individual treatment, rather than the complete
time frame.
Similarly, the patient can be described, assuming the description of
the patient description does not change,
by $\forall s\; P(s)$, where $P$ is a parameterised description
of the patient. For example, in diabetes, it may be assumed
that the Quetelet index does not change; however, the condition
generally does change due to the application of a treatment.

The guideline as shown in Fig.~\ref{db2man} is modelled in two
parts. First, we need to specify which treatment is administered in
each step of the protocol. Second, the transition of one step
to the next has to be specified.
The former is modelled as a conjunction of treatments for each
step of the guideline. For example, in the initial treatment 
step (i.e., step 1) only `diet' is applied, hence, the following
is specified:
\[
\textsf{G}\,\textit{diet}(1)
\]
In general, for treatment $T(s)$ in step $s$, we write $\textsf{G} T(s)$.  Here
$s$ is a meta-variable standing for the step in the protocol, i.e., it is a
ground atom in the concrete specification of the protocol. Object-level
variables can be recognised by the fact that they are bounded by
quantification. For example, $T(s)$ is a ground term in the actual
specification, while $\forall s\;T(s)$ is not. 
In this notation, we will refer to the set of treatment prescriptions for each
step and all patient groups $P(s)$ as $\mathcal{D} = \bigcup_{s} P(s)
\rightarrow \textsf{G}\,T(s)$.

The second part of the formalisation concern the change of
treatments, which is formalised by a predicate $\textit{control}(s)$
that describes which step of the guideline will be reached. 
Recall from Fig.~\ref{db2man}, that
treatments are stopped in case they fail, i.e., when they do not
result in the desired effect. 
This change
of control can be described in the meta-language as:
\begin{equation}
\mathcal{B} \cup \textsf{G}\,T(s) \cup P(s) \not\models N(s) \Rightarrow
\textit{control}(s+1)
\label{eq:controlaxiom}
\end{equation}
for all steps $s$, i.e., if the intention cannot be deduced, then we move to a 
subsequent step. We will refer to this axiom as the \emph{control
axiom} $\mathcal{C}$. It is not required that the control is 
mutually exclusive: if $\textit{control}(s+1)$ holds, then $\textit{control}(s)$
also holds, although the converse is not necessarily true. 
Note that $\neg N(s)$ cannot be deduced from the background
knowledge, due to its causal nature; however, clearly, in the context of automatic
reasoning, it is useful to reason about the theory deductively. To 
be able to do this, one can use the
so-called completed theory, denoted as COMP($\Gamma$), where $\Gamma$
is some first-order theory. The COMP function is formally defined
in \cite{clark78} for general first-order theories. For propositional
theories one can think of this function as replacing implication with
bi-implications, for example, COMP($p \rightarrow q$) = $p \leftrightarrow
q$ and COMP(\{$p \rightarrow q$, $p \rightarrow r$\}) = $p \leftrightarrow 
(q \lor r)$. By the fact that the temporal formulas can be interpreted
as first-order sentences, we have for example:
\[
\begin{array}{ll}
\ \ \ \ &
\textup{COMP}(\textsf{G}\,\textrm{Drug}(\textit{insulin}) \rightarrow
\textsf{G}\,\textit{uptake}(\textit{liver},\textit{glucose}) =
\textit{up}) \\
& = \textsf{G}\,
\textrm{Drug}(\textit{insulin}) \leftrightarrow
\textsf{G}\,\textit{uptake}(\textit{liver},\textit{glucose}) =
\textit{up}
\end{array}
\]
This can be extended for the whole set of axioms of diabetes.
The relevance of this operator for this chapter, is that abductive
reasoning can be seen as deductive reasoning in this completed theory
\cite{console91relationship}. 
In the following section,  we introduce an
extension to this idea for the restricted part of temporal logic
described in Section~\ref{abductive}.  These results are based on a
direct application of work done by St\"ark \cite{staerk94}.  Then, we will apply 
those results to the above formalisation.

\subsubsection{Completion}
\label{section:completion}

An important resolution strategy is SLD resolution which is linear
resolution with a selection function for Horn clauses, i.e., clauses
with at most one positive literal (for a
definition see for example \cite{lucas04}). SLD resolution is sound
and refutation complete for Horn clause logic. It is refutation complete in the sense that
if one would use a breadth-first strategy through the tree of all SLD
derivations, a finite SLD refutation will be found if the set of Horn
clauses is unsatisfiable. Below, as a convenience, we will write that we derive
$\psi$ from $\varphi$ using SLD resolution iff there is an SLD refutation from 
$\varphi \land \neg \psi$.

SLDNF resolution augments SLD resolution with a so-called `negation as
failure' (NAF) rule \cite{clark78}. The idea is in order to prove
$\neg A$, try proving $A$; if the proof succeeds, then the evaluation
of $\neg A$ fails; otherwise, if $A$ fails on every evaluation path,
then $\neg A$ succeeds. The latter part of this strategy is not a
standard logical rule and could be described formally as, given some theory $\Gamma$, if
$\Gamma \not\vdash A$ then $\Gamma \vdash \neg A$ is concluded. It
must be noted that the query $A$ must be grounded. This
type of inference is featured in logic programming
languages such as \textsc{prolog}, although most implementations 
also infer the negation as failure for non-ground goal clauses.

This type of resolution is used here to show that a completed theory can
be used in a deductive setting to reason about the meta-theory. In
particular, in \cite{staerk94}, this is used to show that a certain
class of programs have the property that if a proposition deductively
follows from that program, then there is a successful SLDNF
derivation. This is shown by so-called input/output specifications,
which are given by a set of mode specifications for every predicate. A
mode specification for a predicate says which arguments are
\textsf{input} arguments and which arguments are \textsf{output}
arguments; other arguments are called \textsf{normal} arguments. Given
an input/output specification a program must be written in such a way
that in a computed answer the free variables of the output terms are
contained in the free variables in the input terms. Furthermore, the
free variables of a negative literal must be instantiated to ground
terms during a computation. For example, the following well-known logic program
\begin{eqnarray*}
&&\textsf{append}([],L,L). \\
&&\textsf{append}(L_1,L_2,L_3) \rightarrow \textsf{append}([X|L_1],L_2, [X|L_3]).
\end{eqnarray*}
has two mode specifications.
Either the first two arguments are input arguments resulting in a
concatenation of the two lists in the output argument, or, the first
two arguments can act as output arguments resulting in the
decomposition of the third argument into two lists. 

In the following, we will write all ground atoms without arguments, e.g.,
we denote $A$ when we mean $A(c)$, where $c$ is some constant, unless the
constant is relevant.
We then prove the following lemma.
\begin{lemma}
If $\textup{COMP}(\Gamma) \models \neg A_g$, where 
$\Gamma$ is a formula of the form:
\[
\forall s \forall t \;(A_0(s)\land\cdots\land A_n(s)\land A_{n+1}(t,s)\land\cdots\land
A_{m}(t,s) \rightarrow A_k(t,s))
\]
where $A_i$ are all positive atoms and $A_g$ is any
ground atom, then
there exists an SLDNF derivation of $\neg A_g$ for 
theory $\Gamma$. 
\label{lemma:correctS}
\end{lemma}
A proof can be found in Appendix \ref{appendixa}. Note here
that $\Gamma$ only contains Horn clauses. 
Further note that the relation between the completed theory and SLDNF
derivation holds for a much more elaborate class of
formulas \cite{staerk94}. Hence, this result could be generalised 
to a more elaborate temporal descriptions. 
However, the fact that we are dealing with Horn clauses yields the
following property, which is the main result of this section.
\begin{theorem}
If $\Gamma$ is in the form as assumed in
Lemma \ref{lemma:correctS}, $A$ is again any
ground atom, and it holds that $\textup{COMP}(\Gamma) \models \neg A$, then
$\Gamma \not\models A$.
\label{theorem:completion}
\end{theorem}
\begin{proof}
Suppose $\textup{COMP}(\Gamma) \models \neg A$. 
Then by Lemma \ref{lemma:correctS} it holds that 
$\neg A$ is derived by SLDNF resolution from $\Gamma$.
From the definition of SLDNF derivation either $\neg A$ 
holds by SLD resolution
or a derivation for $A$ fails.
In either way, it follows from the soundness of SLD resolution 
that deriving $A$ from $\Gamma$ using SLD resolution will fail.
Since each of the clauses is Horn and SLD resolution is complete for these
Horn clauses, it follows that  $\Gamma \not\models A$. 
\end{proof}

\subsubsection{Implementation}
\label{section:impl}

The result of Theorem \ref{theorem:completion} is used to investigate 
the completion of a restricted subset
of temporal logic. To simplify matters, we introduce the following 
assumptions. First, the \textsf{H} operator is omitted. In this case,
this is justified as this operator only plays a role to denote the
fact that the patient suffers from hyperglycaemia and plays no
role in the temporal reasoning. Hence, we have a (propositional) variable
that expresses exactly the fact that in the past the condition was hyperglycaemic.
Second,
as there is no reasoning about the past, we may translate $\textsf{G} \varphi$
to $\forall t\; \varphi(t)$. Finally, we only
make a distinction between whether the glucose level is decreasing or not, i.e.,
we abstract from the difference between normo- and hypoglycaemia.
Furthermore, we assume that the mutual exclusion of values for
capacity is omitted and part of the description of the patient, i.e.,
a patient with $\textup{QI} > 27$ is now described by $\{ \textup{QI} > 27, 
\neg (\textup{QI} \leq 27) \}$.
We will refer to these translation assumptions in addition to
the translation to first-order logic described in
Section~\ref{trans.object} as $\textsf{ST}_{t}'$. Furthermore, let
$\textrm{COMP}(\Gamma)$ be understood as the formula which is equivalent
according to $\ST$ to
$\textrm{COMP}(\textsf{ST}_{t}'(\Gamma))$ whenever $\Gamma$ is a theory in
temporal logic.
Note that this abstraction is sound, 
in the sense that 
anything that is proven with respect to the condition of the patient
by the abstracted formulas can be proven from
the original specification.

\begin{figure}
\begin{center}
\begin{tabular}{l@{\hspace{3cm}}l}
\hline\hline
\textbf{Temporal Logic} & \textbf{First-order Logic} \\
\hline
$A_1\land\cdots\land A_n \land \textsf{G}\, A_{n+1} \land$
& $\forall t\; (A_1\land\cdots\land A_n\, \land
A_{n+1}\land$ \\
\ \ $\cdots\land
\textsf{G}\, A_m \rightarrow \textsf{G}\, A_i$ 
& \ \ $\cdots\land A_m
\rightarrow A_i(t))$
\\

$\textsf{G} (A_1 \land\cdots\land A_n \rightarrow A_i)$ 
& $\forall t\; (A_{1}(t)\land\cdots\land A_{n}(t) \rightarrow A_i(t))$
\\
$\textsf{G}\,A_i$, $A_i$ 
& $A_i(t)$, $A_i$
\\
$\neg \textsf{G}\,A_i$ 
& $\neg A_i$
\\
\hline\hline
\end{tabular}
\end{center}
\vspace{-\baselineskip}
\caption{The type of temporal formulas and their translation, where
the Skolem constants describing time instances are omitted.}
\label{figure:clausetypes}
\end{figure}

Let $p_i$ be a patient characteristic, $d$ a drug, and $l_i$
either a patient characteristic or drug. 
The temporal formulas that are allowed are listed in
Fig.~\ref{figure:clausetypes}. 
We claim that each temporal formula is an instance of a temporal
formula mentioned in Fig.~\ref{figure:clausetypes}, universally quantified
by a step $s$, except for the last goal clause which is grounded. The
background knowledge can be written in terms of the first and second clause, 
taken into
account that axiom (7) can be rephrased to two clauses of the first
type and we need to make sure that each literal is coded
as a positive atom. This is a standard translation procedure that can
be done for many theories and is described in e.g.,
\cite[p.~23]{Shepherdson:1987DDLP}. Axiom (3) needs to be rewritten for each of
the cases of capacity implied by the negated sub-formula. For each
drug and patient characteristic in the hypothesis, the third clause of 
Fig.~\ref{figure:clausetypes} applies. A goal is an
instance of the fourth clause of Fig.~\ref{figure:clausetypes}.
As the first three clauses are Horn, Theorem~\ref{theorem:completion} can 
be instantiated for the background knowledge, which yields:
\begin{theorem}
$\textup{COMP}({\cal B}'_{\mbox{\scriptsize DM2}} \cup
\textsf{G}\,T(s) \cup P(s)) \models \neg
N(s)$ implies
$\mathcal{B}'_{\mbox{\scriptsize DM2}} \cup P(s) \cup \textsf{G} T(s) \not\models 
N(s)$.
\label{theorem:completioninstant}
\end{theorem}
This states that, if the completed theory implies that the patient will not 
have normoglycaemia, then this is consistent conclusion with respect to the original
specification, for any specific step described by $s$. Therefore, there is no reason to assume that $T$ is the
correct treatment in step $s$. This result is applied to the control
axiom $\mathcal{C}$ as described in Section~\ref{section:form},
i.e., formula \ref{eq:controlaxiom}.
If we were to deduce that 
\[
\textup{COMP}(\mathcal{B} \cup \textsf{G}\, T(s) \cup P(s)) \models
\neg N(s)\]
then, assuming the literals are in a proper form required by
Theorem \ref{theorem:completioninstant}, this implies that
\[
\mathcal{B} \cup \textsf{G}\,T(s) \cup P(s) \not\models N(s)
\]
Thus, we postulate the following axiom describing the
change of control, denoted by $\mathcal{C'}$
\[
\textup{COMP}(\mathcal{B} \land \textsf{G}\,T(s) \land P(s)) \land
\neg N(s) \rightarrow control(s+1)
\]
The axioms $\mathcal{D}$ (cf. Section~\ref{section:form}) and $\mathcal{C'}$ 
are added to the guideline formalisation in order to reason about the
structure of the guideline.

To investigate the quality of the treatment sequence, a choice
of quality criteria has to be chosen. Similarly to individual treatments,
notions of optimality could be studied. Here, we investigate 
the property that for each patient group, the intention should be
reached at some point in the guideline.
For the diabetes guideline, this is formalised as follows:
\[
\mathcal{B}'_{\mbox{\scriptsize DM2}} \cup \mathcal{D} \cup \forall s\;P(s) \models
\exists s\; N(s)
\]
As we restrict ourselves to a particular treatment described
in step $s$, this property is similar to the property proven in
Section~\ref{implementation}.
However, it is possible that the control never reaches $s$ for
a certain patient group, hence,
using the knowledge described in $\mathcal{C}$, it is also 
important to verify that this step is indeed
reachable, i.e.,
\[
\mathcal{B}'_{\mbox{\scriptsize DM2}} \cup \mathcal{D} \cup \forall s\;P(s)) \cup
\mathcal{C'} \models
\exists s\; (N(s) \land \textit{control}(s))
\]

The above was used to verify a number of properties
for different patient groups. For example, assume \[
\begin{array}{ll}
P(s) = & \{
\textit{capacity}(\textit{liver},\textit{glucose},s) =
\textit{exhausted}, \textup{QI}(s) \leq 27, \\
& \textsf{H}\,\textrm{Condition}(\textit{normoglycaemia}) \}
\end{array}
\] 
(note
the \textsf{H} operator is abstracted from the specification) then:
\[
\mathcal{B}'_{\mbox{\scriptsize DM2}} \cup \mathcal{D} \cup \forall_{s}\;P(s) \cup
\mathcal{C'}  \models \textsf{G}\,\textrm{Condition}(\textit{normoglycaemia},3)
\land \textit{control}(3)
\]
i.e., the third step will be reached and in this step the patient will
be cured. This was implemented in \otter{} using the translation as discussed in
the previous subsection. As the temporal reasoning is easier due to
the abstraction that was made, the
proofs are reasonably short. For example, in the example above, the proof 
has length 25 and
was found immediately.

\section{Conclusions}
\label{conclusions}

The quality of guideline design is for the largest part based on its
compliance with specific treatment aims and global requirements.  We
have made use of a logical meta-level characterisation of such
requirements, and with respect to the requirements use was made of the
theory of abductive, diagnostic reasoning, i.e., to diagnose potential
problems with a guideline \cite{lucas97,lucas2003,poole90}.  In
particular, what was diagnosed were problems in the relationship
between medical knowledge, and suggested treatment actions in the
guideline text and treatment effects; this is different from
traditional abductive diagnosis, where observed findings are explained
in terms of diagnostic hypotheses. This method allowed us to examine
fragments of a guideline and to prove properties of those fragments.
Furthermore, we have succeeded in proving a property using the
structure of the guideline, namely that the blood glucose will go
down eventually for all patients if the guideline is followed
(however, patients run the risk of developing hypoglycaemia).

In earlier work \cite{HommersomIEEE2007}, we used a tool for interactive
program verification, named KIV \cite{reif95}, for the purpose of
quality checking of the diabetes type 2 guideline. Here, the main
advantage of the use of interactive theorem proving was that the
resulting proofs were relatively elegant as compared to the solutions
obtained by automated resolution-based theorem proving. This may be
important if one wishes to convince the medical community that a
guideline complies with their medical quality requirements and to
promote the implementation of such a guideline.  However, to support
the \emph{design} of guidelines, this argument is of less importance.
A push-button technique would there be more appropriate.  The work
that needs to be done to construct a proof in an interactive theorem
prover would severely slow down the development process as people with
specialised knowledge are required. 

Another method for verification that is potentially useful is model
checking. One advantage is that it allows the end user, in some cases,
to inspect counter example if it turns out that that a certain quality
requirement does not hold. The main disadvantage is that the domain
knowledge as we have used here is not obviously represented into a
automaton, as knowledge stated in linear temporal logic usually
cannot succinctly be translated to such a model.

One of the main challenges remains bridging the gap between guideline
developers and formal knowledge needed for verification.  The
practical use of the work that is presented here depends on such
developments, although there are several signs that these developments
will occur in the near future. Advances in this area have been made in
for example visualisation \cite{KM2001} and interactive execution of
guideline representation languages.  Furthermore, the representation
that we have used in this paper is conceptually relatively simple
compared to representation of guidelines and complex temporal
knowledge discussed in for example \cite{shahar00modelbased}, however,
in principle all these mechanisms could be formalised in first-order
logic and could be incorporated in this approach. 
Similarly, probabilities have been ignored in this paper, for which
several probabilistic logics that have been proposed in the last
couple of years seem applicable in this area 
\cite{Richardson:2006ML,Kersting:2000ILP}. Exploring other types
of analysis, including quantitative and statistical, could
have considerable impact,
as we are
currently moving into an era where guidelines are evolving into highly
structured documents and are constructed more and more using
information technology. It is not unlikely that the knowledge itself
will be stored using a more formal language. Methods for
assisting guideline developers looking into the quality of clinical 
guidelines, for example, using automated verification will then be useful.

\subsubsection*{Acknowledgement}
This work was partially supported by the European Commission's IST
program, under contract number IST-FP6-508794 (Protocure II project).

\appendix

\section{Proof of Meta-level Property (M2)}
\label{firstproof}

In the formulas below, each literal is augmented with a time-index.
These
implicitly universally quantified variables are denoted as $t$ and
$t'$. Recall that $g(x,y) = \textit{down}$ is implemented as $\neg
(g(x,y) = \textit{up})$ and functions $f$ and $f'$ are Skolem
functions introduced by \otter. Both Skolem functions map a time point
to a later time point. Consider the following clauses in the usable
and set-of-support list:

\begin{description}
\item[\texttt{ \ \ \ 2}] $\textit{capacity}(\textit{b-cells},
\textit{insulin},t) \neq
\textit{nearly-exhausted}\;\lor$\\
$\textit{\ \ \ capacity}(\textit{b-cells},
\textit{insulin}, t) \neq \textit{exhausted}$

\item[\texttt{ \ \ 14}] $t \not> f(t) \lor
  \textit{capacity}(\textit{b-cells},\textit{insulin},t) = 
  \textit{exhausted}\lor t > t'\, \lor$\\
  $\textit{\ \ \ secretion}(\textit{b-cells},
  \textit{insulin}, t') = \textit{up}$

\item[\texttt{ \ \ 15}] $\neg \textup{Drug}(\textup{SU}, f(t)) \lor
  \textit{capacity}(\textit{b-cells},\textit{insulin}, t) = 
  \textit{exhausted} \lor t > t'\, \lor$\\
  $\textit{\ \ \ secretion}(\textit{b-cells}, \textit{insulin}, t') = \textit{up}$
\item[\texttt{ \ \ 51}] $0 > t \lor \textup{Drug}(\textup{SU},t)$
\item[\texttt{ \ \ 53}] $
\textit{capacity}(\textit{b-cells},\textit{insulin}, 0) =
  \textit{nearly-exhausted}$
\end{description}
For example, assumption (53)
models the capacity of the B~cells, i.e., nearly exhausted at time $0$ 
where the property as shown above should be refuted. Note that
some of the clauses are introduced in the translation to propositional
logic, for example assumption (2) is due to the fact that 
that values of the capacity are mutually exclusive. This is
consistent
with the original formalisation, as functions map to unique elements
for element of the domain.

Early in the proof, \otter{} deduced that if the capacity
of insulin in B~cells is nearly-exhausted, then it is not completely
exhausted:
\begin{description}
\item[\texttt{ \ \ 56}] \texttt{[neg\_hyper,53,2]} $ 
  \textit{capacity}(\textit{b-cells},\textit{insulin}, 0) 
  \neq \textit{exhausted}$
\end{description}
Now we skip a part of the proof, which results in information about
the relation between the capacity of insulin and the secretion of 
insulin in B~cells at a certain time point:
\begin{description}
\item[\texttt{ \ 517}] \texttt{[neg\_hyper,516,53]} $0 \not>
f'(0)$
\item[\texttt{ \ 765}] \texttt{[neg\_hyper,761,50,675]} \\
  $\textit{\ \ \ \ \ \ \ \ \ \ \ \ capacity}(\textit{b-cells},
  \textit{insulin}, f'(0))
  \neq \textit{nearly-exhausted} \lor \\
  \textit{\ \ \ \ \ \ \ \ \ \ \ \ secretion}(\textit{b-cells},
  \textit{insulin}, f'(0)) =
  \textit{down} $
\end{description}

This information allows \otter{} to quickly complete the proof,
by combining it with the information about the effects of a 
sulfonylurea drug:
\begin{description}
\item[\texttt{766}] \texttt{[neg\_hyper,765,15,56,517]} \\
$\textit{\ \ \ \ \ \ \ \ \ \ \ \ capacity}(\textit{b-cells}, \textit{insulin}, f(0)) \neq
\textit{nearly-exhausted}\;\lor$\\ 
$\textit{\ \ \ \ \ \ \ \ \ \ \ \ }\neg\textup{Drug}(\textup{SU},
f'(0))$
\item[\texttt{767}] \texttt{[neg\_hyper,765,14,56,517]} \\
$\textit{\ \ \ \ \ \ \ \ \ \ \ \ capacity}(\textit{b-cells}, \textit{insulin}, f(0)) \neq
\textit{nearly-exhausted}\;\lor $\\
$\textit{\ \ \ \ \ \ \ \ \ \ \ \ }
 0 \not> f(0)$
\end{description}
after which (53) can be used as a nucleus to yield:
\begin{description}
\item[\texttt{ 768}] \texttt{[neg\_hyper,767,53]} $0 \not> f_1(0)$
\end{description}
and consequently by taking (51) as a nucleus, we find that at time
point 0 the capacity of insulin is not nearly exhausted:
\begin{description}
\item[\texttt{ 769}] \texttt{[neg\_hyper,768,51,766]} $ \\
\textit{\ \ \ \ \ \ \ \ \ \ \ \ capacity}(\textit{b-cells},\textit{insulin}, 0) \neq
\textit{nearly-exhausted}
$
\end{description}
This directly contradicts one of the assumptions and this results in 
an empty clause:
\begin{description}
\item[\texttt{ 770}] \texttt{[binary,769.1,53.1]} $\bot$
\end{description}

\section{Proof of Lemma 1}
\label{appendixa}

Let $\Gamma$ and $\Pi$ denote lists of literals. An $n$-tuple
$(x_1,\ldots,x_n) \in
\{\textsf{in},\textsf{out},\textsf{normal}\}^{n}$ is called
a \emph{mode specification} for an $n$-place relation symbol $R \in
Rel$, denoted by $\alpha,\beta,\gamma$. The set of \emph{input variables} of the atom
$R(t_1,\ldots,t_n)$ (where $t_i$ is a term) given a mode specification 
is defined by:
\[
in(R(t_1,\ldots,t_n), (x_1,\ldots, x_n)) = \bigcup\{ vars(t_i) \mid
1 \leq i \leq n, x_i = \textsf{in} \}.
\]
Analogously, the set of \emph{output variables} is given by
\[
out(R(t_1,\ldots,t_n), (x_1,\ldots, x_n)) = \bigcup\{ vars(t_i) \mid
1 \leq i \leq n, x_i = \textsf{out} \}.
\]
An input/output specification is a function $S$ which assigns to
every $n$-place relation symbol $R$ a set $S^{+}(R/n) \subseteq
\{\textsf{in},\textsf{out},\textsf{normal}\}^{n}$ of positive
mode specification and a set $S^{-}(R/n) \subseteq
\{\textsf{in},\textsf{normal}\}^{n}$ of negative mode 
specifications for $R$.

\begin{definition}[Definition 2.1 \cite{staerk94}]
A clause $\Pi \rightarrow A$ is called \emph{correct} with respect to an
input/output specification $S$ or \emph{S-correct} iff
\begin{description}
\item[(C1)] for all positive modes $\alpha \in S^{+}(A)$ there exists
a permutation of the literals of the body $\Pi$ of the form $B_1,
\ldots, B_m, \neg C_1, \ldots \neg C_n$ and for all $1 \leq i \leq m$
a positive mode $\beta_i \in S^{+}(B_i)$ such that
\begin{itemize}
\item for all $1 \leq i \leq m$, $in(B_i,\beta_i) \subseteq 
in(A,\alpha) \cup \bigcup_{1 \leq j \leq i} out(B_j, \beta_j)$,
\item $out(A,\alpha) \subseteq in(A,\alpha) \cup \bigcup_{1 \leq
j\leq m} out(B_j, \beta_j)$,
\item for all $1 \leq i \leq n$, \\
$S^{-}(C_i) \neq \varnothing$ and $vars(C_i) \subseteq in(A,\alpha)
\cup \bigcup_{1 \leq j \leq m} out(B_j, \beta_j)$,
\end{itemize}
\item[(C2)] for all negative modes $\alpha \in S^{-}(A)$ for all
positive literals $B$ of $\Pi$ there exists a negative mode $\beta
\in S^{-}(B)$ with $in(B,\beta) \subseteq in(A,\alpha)$ and for all
negative literals $\neg C$ of $\Pi$ there exists a positive mode
$\gamma \in S^{+}(C)$ with $in(C,\gamma) \subseteq in(A,\alpha)$.
\end{description}
\end{definition}
A program $P$ is called \emph{correct} with respect to an input/output
specification $S$ iff all clauses of $P$ are $S$-correct.

\begin{definition}[Definition 2.2 \cite{staerk94}]
A goal $\Gamma$ is called \emph{correct} with respect to an
input/output specification $S$ or \emph{S-correct} iff there exists a
permutation $B_1,
\ldots, B_m, \neg C_1, \ldots \neg C_n$ of the literals of $\Gamma$
and for all $1 \leq i \leq m$ a positive mode $\beta_i \in S^{+}(B_i)$
such that
\begin{description}
\item[(G1)] for all $1 \leq i \leq m$, $in(B_i, \beta_i) \subseteq
\bigcup_{1 \leq j \leq i} out(B_j, \beta_j)$,
\item[(G2)] for all $1 \leq i \leq m$, $S^{-}(C_i) \neq \varnothing$
and $vars(C_i) \subseteq \bigcup_{1 \leq j \leq m} out(B_j, \beta_j)$.
\end{description}
\end{definition}

\begin{theorem}[reformulation of Theorem 5.4 \cite{staerk94}]
Let $P$ be a normal program which is correct with respect to the
input/output specification $S$ and let $L_1, \ldots, L_r$ be a
goal.
\begin{description}
\item[(a)] If $\textup{COMP}(P) \models L_1 \land \ldots \land
l_r$ and $L_1, \ldots, L_r$ is correct with respect to $S$
then there exists a substitution $\theta$ such that there is a
successful SLDNF derivation for $L_1,\ldots,L_r$
with answer $\theta$. (...)
\end{description}
\end{theorem}

Define $S^{+} = S^{-} = \{ (\textsf{normal} \}$
for every unary predicate and $\{ (\textsf{normal}, \textsf{normal}) \}$ for every
binary predicate. Observe that $\Gamma$ contains only definite clauses,
so each condition in Definition 1 is trivially satisfied, thus $\Gamma$
is $S$-correct. Similarly,
as the goal $\psi$ is definite, all clauses of Definition 2 are
trivially satisfied, thus also $S$-correct.  Hence, by
Theorem 3, we find that there is a successful SLDNF derivation of
$\psi$ given $\Gamma$.

\bibliographystyle{acmtrans}
\bibliography{otter}

\end{document}